\documentclass{article} 

\usepackage{latexsym}
\usepackage{amssymb}
\usepackage{amsmath}
\usepackage{amsthm}
\usepackage{booktabs}
\usepackage{enumitem}
\usepackage{graphicx}
\usepackage{color}
\usepackage{tikz}
\usepackage{orcidlink}
\usepackage{a4wide}

\newtheorem{defi}{Definition}

\usepackage{amsfonts}
\usepackage{amssymb}
\usepackage{graphicx}
\usepackage{multirow}
\newtheorem{ex}{Example}
\newtheorem{prop}{Proposition}

\newtheorem{remark}{Remark}
\newtheorem{corollary}{Corollary}

\newenvironment{exs}[1]{\smallskip \noindent{\bf Example~\ref{#1}   (continued):}\em}{}
\newcommand\Omit[1]{}
\DeclareMathOperator{\entails}{\mid\!\sim}
\DeclareMathOperator{\entailsnot}{\mid\!\not\sim}

\DeclareMathOperator*{\inn}{\mathrel{{\subset}\!\!\!\!{\Equal}}}
\DeclareMathOperator*{\CWA}{CWA}
\DeclareMathOperator*{\NP}{NP}
\DeclareMathOperator*{\lollypop}{\rule[2.5pt]{6pt}{.5pt}\hspace{-1pt}\circ}
\DeclareMathOperator*{\SMAX}{SMax}

\newcommand\todoin[2][]{\todo[inline, caption={2do}, #1]{
\begin{minipage}{\textwidth-4pt}#2\end{minipage}}}

\usepackage{tikz}
\usetikzlibrary{arrows,shapes}

\newcommand{\flo}[1]{\textcolor{purple}{Flo: #1}}
\newcommand{\ben}[1]{\textcolor{blue}{Ben: #1}}

\renewcommand{\L}{\mathcal L}
\newcommand{\V}{\mathcal V}
\newcommand{\G}{\mathcal G}
\newcommand{\until}[1]{\!_{\rightarrow #1}}
\def\Equal{\rule[.4ex]{5pt}{0.4pt}\llap{\rule[.7ex]{5pt}{0.4pt}}}
\errorstopmode
\begin{document}

\title{What killed the cat? \\Towards a logical formalization of curiosity (and
  suspense, and surprise) in narratives\thanks{This is a draft version, the
    paper is published in LIPIcs, Volume 318, containing the proceedings of the 31st International Symposium on
    Temporal Representation and Reasoning (TIME 2024).}} 



\author{\hspace{-0.5cm}\begin{tabular}{llll}
Florence Dupin de Saint-Cyr$^{1,2}$ & Anne-Gwenn Bosser$^1$ & Benjamin Callac$^1$ & Eric
Maisel$^1$
\end{tabular}\\
  \small $^1$ Lab-STICC\_COMMEDIA, CNRS UMR6285, Brest, France \\
  \small $^2$ IRIT, Université de Toulouse, France \\
}

\maketitle

\begin{abstract}
We provide a unified framework in which the three emotions at the
heart of narrative tension (curiosity, suspense and surprise) are formalized. This
framework is built on non-monotonic reasoning which
allows us to compactly represent the default behavior of the world and to
simulate the affective evolution of an agent receiving a story.
After formalizing the notions of awareness, curiosity, surprise and suspense, we
explore the properties induced by our definitions and study the computational
complexity of detecting them. We finally propose means to evaluate these
emotions’ intensity for a given agent listening to a story.

\textbf{Keywords:} Knowledge Representation, Narration, Cognition 
\end{abstract}

\section{Introduction}
\label{sec:intro}

Humans tell stories to make sense of the world and communicate  their understanding of what happens. Storytelling supposes to be able to sort out which events are worth telling, deciding on a level of detail for describing events, selecting among possible causes the ones which are deemed worth telling. It also supposes to make use of an affective machinery for capturing an audience's attention (emotional contagion, suspense elicitation...). In the act of storytelling, structural and affective phenomena are thus combined with communicative goals in mind. This combination has indeed shown its effectiveness in this respect: the phenomenon of narrative transportation (the experience of being immersed in a story) has been linked to persuasion~\cite{green_2000}. The narrative paradigm therefore provides an appropriate framework, in which causal reasoning about the situations narrated~\cite{trabasso_1985} is combined with narrative devices to encourage the audience's emotional involvement~\cite{sternberg_2001}, to study and model how opinion is formed and evolves. Building a framework for reasoning about and unveiling storytelling mechanics could pave the way for intellectual self-defense supporting tools, enabling citizens to arm themselves against hostile disinformation or influence campaigns.

  
Previous works in structural narratology have studied
the way stories are conveyed to their audience and seminal work from (for
instance) Genette \cite{genette1979} or Propp \cite{propp_1968} have previously served as the
backbone inspiration for computational narrative models and storytelling systems
\cite{pizzi_2007}. Whilst the operationalization of narrative theories is still
subject to debate and caution, such works have shed light on how the story
material to tell 
and the manner in which it is told 
interacts with a model of the listener\footnote{For sake of homogeneity, we use the term listener in all the paper, while this kind of agent is called interpret by Baroni.} (which, depending on the media used for conveying the story can also be a reader, a spectator, or even a gamer): the act of storytelling can thus be understood as knowledge transfer and manipulation of her beliefs.

According to Sternberg \cite{sternberg_2001} or Baroni \cite{baroni_2007}, emotions more specific to narratives which are \emph{suspense}, \emph{curiosity} and \emph{surprise} are critical to retain the interest of the listener. Drivers of the \emph{narrative tension}, they are paramount in maximizing her engagement. 
%
%
%
In this paper we focus on these narrative tension's building blocks. 

In the field of computational narratives, numerous studies and frameworks exist to tell interactive stories, a number of them as an application of planning technologies \cite{Rivera_Jhala_Porteous_Young_2024} allowing to adapt the narrative to each user's actions.
However, adapting a narrative to a model of the user's emotions remains largely a challenge that needs to be addressed to favor engagement: narrative engagement depends partly on the appropriate maintenance of narrative tension, itself based on the uncertainty occurring in a narrative \cite{bonoli_2008}, 
and listener's models based on a formalization of related emotions have comparatively been less addressed so far in the literature. 
While suspense and surprise have been the object of previous studies \cite{cheong_2015} \cite{ely_2015}, there is --- to our knowledge --- still no curiosity model applicable to narratives.



In the following, we present a preliminary study for characterizing these emotions from an epistemic standpoint, with a focus on modeling the listener's curiosity depending on her beliefs and knowledge using a propositional language. Our overarching aim is to provide a unifying framework allowing to represent emotions relevant to the characterization of narrative tension and its evolution, which would enable to discuss their relationships and ultimately help establish dramatic metrics about a narrative.

In Section~\ref{sec:state_of_art}, we describe the main emotions supporting narrative tension. We also describe the problems and solutions for formalizing reasoning about action and change, as well as the ways in which the notions of surprise and awareness have been treated in the literature. Section~\ref{sec:formalization} details our proposal, which relies on a non-monotonic framework resulting from extending propositional logic with default rules. The properties of the framework are presented in Section~\ref{sec:properties}, along with some preliminary ideas for developing metrics. 




\section{Background on reasoning about change and narrative tension}\label{sec:state_of_art}

In order to reason about a story, it is useful to dispose of a way to handle the concepts involved for understanding a sequence of facts and events. We first briefly outline the background to this vast subject, before discussing the formalization of emotions in the literature.
\subsection{Reasoning about action and change}
The formalization of action and change is an old field of research in the domain of knowledge representation and reasoning in AI.
There are many different reasoning tasks in this field (see e.g. \cite{DHLM20}) like prediction of the new state of the world after an action (which is related to 
\emph{belief update} \cite{Winslett90, KaMe91}), or integration of an observation (which is related to \emph{belief revision} \cite{AGM85}), or \emph{event abduction} which consists in guessing which event took place, or \emph{scenario extrapolation} \cite{DuLa11,DupindeSaintCyr08} which consists in taking a partial description of facts and events that occurred and complete it (by prediction or event abduction) or \emph{scenario recognition} \cite{DoLM07}. 

These reasoning tasks were studied in various frameworks, the representation of actions in a compact way has given rise to some problems known as the frame, the ramification and the qualification problems \cite{McHa69,Finger87,McCarthy77}. 
%
In propositional logic, these problems were solved by a majority of approaches by introducing a special symbol for expressing a causal rule relating preconditions of an action to its effect (indeed classical implication cannot separate a cause from a consequence due to contraposition). Actions are first described by such rules, then given the set of causal rules, a set of formulas (called frame axioms) are generated stating  that any fluent $f$ is true at time $t+1$ if and only if it was (a) true at $t$ and no causal rule concluding $\neg f$ can be fired at $t$ or (b) false at $t$ and a causal rule concluding $f$ can be fired. 

As a proof of concept, we choose to use \emph{propositional logic} in this article where we face a problem that can be viewed as an extension of belief extrapolation with narrative tension analysis. However in order to perform non-monotonic reasoning (which allows agents to change their minds and thus accept surprises), we use default rules of the form $a \leadsto b$ to encode causal relations. Note that another prominent formalization of default rules was given by Reiter in \cite{Reiter80}, but we choose to rely on a simpler formalism at first.

Logical approaches to computational narratives have been proposed in the past. In ~\cite{bosser_linear_2010}, (Intuitionistic) Linear logic has been argued to be a suitable representational model for narratives for its capacity to finely represent narrative actions through the production and consumption of resources. This language provides the symbol $\lollypop$ which can be used  in $A \lollypop B$ to express the validity of transforming resource $A$ into resource $B$, the flow of resources consumption through the associated sequent calculus allowing to establish causal relations. Dynamic logic \cite{Hintikka62} and its epistemic extensions \cite{bolander2011epistemic} are formalisms with higher expressiveness. In this work, we propose to characterize narrative tension phenomenon in \emph{propositional logic} (extended with default rules) to demonstrate the representational uniformity of these concepts and their relationships with each other. We will explore how to encode them in aforementioned logics in the future, keeping in mind the challenges raised by their operationalization in their most expressive fragments \cite{aucher2013undecidability}. 

\subsection{Emotions supporting narrative tension}

Psychological models of emotions are often used in the field of affective computing such as models from Ekman \cite{ekman_1992} or Plutchik \cite{plutchik_1980} (which include surprise). Other works \cite{ortony_2022} consider that every emotion should have a valence and, as a consequence, surprise, which is inherently neither good nor bad, is considered as a different affective phenomenon.
As we position ourselves in the context of studying the emotional states of a listener, we will rely on the characterizations given by Baroni in \cite{baroni_2007}.
Curiosity occurs when there is a partial omission of crucial knowledge: at a given stage when experiencing the storytelling experience, the listener knows they are missing important information. Suspense arises when an event could potentially lead to an impacting result --- be it good or bad --- to the storyline, and is correlated with anticipation. Surprise results from a rupture from previous expectations, which retroactively invalidates some of the predictions made by the listener:  the listener has expectations about how the story will develop, based on story genre or common sense. Going against these expectations while maintaining coherence is what causes surprise. Baroni distinguishes curiosity and suspense from surprise, as the former two are tied to anticipation and an urge to know, whilst the latter arises sporadically as the narrative progresses, which will be reflected in our model.

Related to suspense, the concept of \emph{narrative closure} also reflects the epistemic nature of storytelling (as theorized by Carroll \cite{carroll_2007}): this encompasses the phenomenological feeling of finality that is generated when all the questions saliently posed by the narrative to the listener are answered. Previous work in the psychology of narrative understanding \cite{trabasso_1985} has also tied the perception of the importance of story events to causal relationships' perception. 
In this paper we borrowed from this work, especially tracing a graph representing the narrative with nodes being actions, preconditions and effects and edges being causal relations. We will consider the degree of a node as reflecting its importance in the narrative, reading it as, the more an action is a consequence and has consequences the more important it is. 

\subsection{Logical models of emotions, surprises and awareness}


The logical representation of emotions has already received some attention, see e.g. Lorini \cite{lorini_2011} or Adam \cite{adam_2009} who formalized emotions based on the OCC theory \cite{ortony_2022}. In these works (which relies on a modal logic for BDI agents), an agent has beliefs, including beliefs about what is \emph{good for herself}, and expresses different emotions such as joy or sadness. 

The particular case of \emph{surprise} was studied by several authors in computer science, but the first study is due to an economist named Shackle \cite{Shackle1961} who defined the degree of surprise associated with an event as the degree of impossibility of this event given the uncertain knowledge about the situation considered. In Lorini and Castelfranchi \cite{LoCa2007}, the role of surprise is investigated in the context of belief update. They associate a surprise with 
a difficulty to integrate the new piece of information, this occurs when there is a form of inconsistency between expectation and perception. 
Surprise was recently formalized in the context of the analysis of jokes by \cite{DuPr23}, indeed surprise has been considered as an important ingredient for laughter by many authors, the model of surprise of \cite{DuPr23} is based on a revision operator and non-monotonic reasoning: to be surprised the listener of a joke should be able to jump to conclusions that can be questioned and even revised.



The characterization of \emph{curiosity} provided by Baroni 
emphasizes that the listener is aware of its incomplete knowledge and that surprise is linked to a notion of disturbance which makes the agent to question his assumptions/beliefs and leads him to reconsider his understanding of the story. This reconsideration reminds the operation of \emph{awareness raising} introduced by \cite{vanbenthem_2010} to allow agents ``to make their implicit knowledge explicit’’. Logical models taking into account agent's awareness  have previously been defined in the literature.
As Halpern \cite{halpern_2001} states, traditionally when reasoning about agents' beliefs, it is assumed they are aware of every proposition. Modica and Ristichini \cite{modica1994awareness} first came up with a definition of awareness based on knowledge, stating that an agent is aware of $p$ if he knew $p$ or if he knew he did not know $p$. Halpern extends on this by introducing \textit{implicit} knowledge, where agents are aware of all propositions and can reason with them ; and \textit{explicit} knowledge, which captures the conclusions of which the agent is explicitly aware of. In this system, explicit knowledge is also implicit, while the reverse is not necessarily true.

Previous works have proposed models for agents in computational narratives such as \cite{norling_2003} or \cite{cardona-rivera_2012} based on Belief, Desire and Intention (BDI). In
\cite{rivera-villicana_2016a}, a BDI agent aiming to simulate player behavior in interactive stories takes into account the player personality. Other work has assigned personality stereotypes to users \cite{thue_2006, barber_2010} according to their interactions with the system. Whilst such models allow personalizing an interactive narrative, they would not enable a storytelling engine to finely drive the narrative tension.
By contrast, the \emph{Suspenser} system by \cite{cheong_2015} offers an
operationalization for suspense elicitation, one of the three drivers of narrative tension. In \emph{Suspenser}, suspense is maximized by ordering multiple story bits at the discourse level. 
We lay in this paper the groundwork for ultimately representing in a unified logical framework suspense, curiosity and surprise, the three drivers of narrative tension. This will build strong foundations for future generative and interactive systems able to operate both at the story and discourse levels.

We approach the modelization of curiosity, suspense and surprise as constructs at given moments of a narrative experience from the listener's  beliefs and by non-monotonic reasoning about these beliefs.
Doing so, we believe our model is compatible with previous formalization while providing new insights.

\section{Formalizing curiosity, surprise and suspense}\label{sec:formalization}

We first present an example to illustrate the concepts introduced throughout this article.
\begin{ex}[The box]\label{exboite}
To illustrate the framework, we 
present a short story involving three agents, Albert, Erwin as well as a protagonist Cecilia\footnote{We consider a story involving Albert Einstein, Erwin Schrödinger and Cecilia Payne-Gaposchkin, hence the cat in the title.} (respectively agents A, E and C). A short narrative : “Cecilia enters her office. She sees a box lying on her desk that was not there when she last left the room.” Our hypothesis is that this event sparks curiosity in Cecilia’s mind. 
We look at it from the point of view of Cecilia who reasons in a closed world where nothing, except three particular events (Albert putting a box on Cecilia's desk, Erwin doing it, Cecilia opening the box) can interact with the state of the world. 
\end{ex}

We consider a set of variable symbols $\V$ denoted by Latin lower case letters, from this set of symbols we build the vocabulary $\V_T$ containing all variables of $\V$ indexed by all the integers taken in the set $T=\{0, 1, .., N\}$ representing time points. $\L$ is the propositional language based on $\V_T$ with the usual connectors and constants $\vee$, $\wedge$, $\neg$, $\to$, $\equiv$, $\bot$ and $\top$ denoting respectively the logical connectors ``or'', ``and'', ``not'', material implication and logical equivalence, contradiction, and tautology. 
The symbol $\models$ represent satisfiability. Let $\Omega$ denote the set of interpretations induced by $\V_T$, we will often use $\omega$ for naming a particular interpretation in $\Omega$, each interpretation will be described by the list of literals satisfied by it, e.g., considering the vocabulary $\V=\{a,b\}$, and a set of two time points $T=\{0,1\}$ $\omega=(a_0, ~\neg b_0, ~\neg a_1, \neg b_1)$ is an interpretation in $\Omega$ that associates the truth value True to $a$ and False to $b$ at time step 0 and False to $a$ and $b$ at time step 1. 
%
%
The set $Mod(A)\subseteq \Omega$ is the set of interpretations satisfying the set of propositional formulas $A\subseteq \L$ ($Mod(A)=\{\omega\in \Omega| \omega \models \bigwedge_{\varphi\in A} \varphi\}$), the same notation is used to represent the set of models of a formula $Mod(\varphi)=\{\omega\in \Omega|\omega \models \varphi\}$. 

\begin{exs}{exboite}
To study this flow of events taking place in 4 time steps denoted $T=\{0,1,2,3\}$ we need a vocabulary $\V=\{box$, $A$, $E$, $C$, $empty$, $visible\}$ meaning respectively there is a box on Cecilia's desk, agent A puts a closed box on the desk, agent E puts a closed box on the desk, agent C opens the box, there is nothing in the box and something inside the box is uncovered (and thus the box has been opened). In the language $\L$ built on $\V$ and $T$, the following expression is an example of a well-formed formula: $(A_0\vee E_0) \wedge box_1 \wedge C_2\wedge \neg empty_2$.
\end{exs}

\emph{Default rules} are rules that tolerate exceptions and allow us to reason in presence of incomplete information, by assuming that the situation is not exceptional when there is no evidence for the contrary. The notation  $\alpha \leadsto \beta$ (with $\alpha,\beta\in \L$) is used to represent a default rule interpreted as when $\alpha$ is true, it is more plausible that $\beta$ is true than false.

\begin{exs}{exboite}
In order to be able to encode this example we propose to use default rules to express that by default some fluents keep their value: the following rule is expressing that when there is no box at time point 0 then by default there is no box at time point 1: $\neg box_0\leadsto\neg box_1$. This rule admits exceptions: namely, if A puts a box on the desk at time point 0 then generally there is a box at time point 1: $A_0\land\neg box_0\leadsto box_{1}$.
\end{exs}

Given a set of default rules $\Delta$ it is possible to define a ranking of these rules according to their specificity, thanks to ``System Z’’ algorithm \cite{Pearl1990}, the default base is then called stratified, its stratas are the formulas with the same rank\footnote{System Z ordering method is based on the
tolerance notion between rules. More precisely, a rule $r=\alpha\leadsto \beta$ is tolerated by a set of $n$ rules $R\subseteq\Delta$ iff $\alpha \wedge\beta \wedge
\bigwedge_{\alpha_i\leadsto\beta_i\in R} (\neg \alpha_i\vee \beta_i)$ is consistent. 
The process continues until $\Delta$ contains only rules tolerated by all the other ones, they constitute the most specific stratum called $\Delta_1$ ($\Delta_n$ being the least specific stratum, with $n$ being the number of iterations).}. Note that there are sets of default rules that do not admit a Z ordering, such default sets are called “inconsistent” in \cite{GoPe91}. In this paper, we restrict ourselves to consistent default sets. From a stratified default base lexicographic-entailment \cite{BCDLP93,Lehmann95} is a non-monotonic inference relation which imposes that the more specific the rules, the more mandatory it is to comply with them:

\begin{defi}[Lex-inference \cite{BCDLP93}]\label{defLex}
Let $\Delta=\Delta_1\cup \cdots \cup \Delta_n$ be a stratified default base with $n$ strata ordered from the most specific strata $\Delta_1$ to the least specific one $\Delta_n$, and let $A$ and $B$ be two subsets of $\Delta$, and $\alpha$, $\beta$ be two formulas of $\L$, 
\begin{itemize}
\item Notations: $str$ (for ``strict’’) is a function that translates a set of default rules into a set of formulas of $\L$, i.e., $str(A)=\bigcup_{\alpha\leadsto\beta\in A}\{\neg \alpha \vee\beta\}$. For all $i\in[1,n]$, and any $E\subseteq \Delta$, $E_i$ denotes the $i$th strata of $E$: $E_i=E\cap \Delta_i$. 
\item $A$ is \emph{Lex-preferred} to $B$ given $\Delta$, denoted $A \succ_\Delta B$, \\ 
\begin{tabular}{p{5.3cm}r}
&iff there exists $k\in [1,n]$ s.t. $\left\{\begin{array}{l}
|A_k|>|B_k| \mbox{ and}\\
\forall i < k,  |A_i|=|B_i|
\end{array}\right.$\end{tabular}
\item $A$ is a \emph{Lex-preferred} $\alpha$-consistent subbase of $\Delta$ if
$A\subseteq \Delta$ and $str(A)\cup \{\alpha\}\not\models \bot$ and for any $B\subseteq\Delta$ s.t. $str(B)\cup \{\alpha\}\not\models \bot$, $B\not\succ_\Delta A$ holds
\item $\alpha \entails_{\Delta} \beta \mbox{ iff for any Lex-preferred $\alpha$-consistent subbase } B$ of $\Delta, str(B)\cup \{\alpha\}\models \beta$
\end{itemize}
\end{defi}

\begin{exs}{exboite} Let us consider that the common knowledge about the world consists only in the default persistence of the fluents $box$, $empty$ and $visible$ and on the default effects of the occurrences of events $A$, $E$ and $C$ when their preconditions hold.\\
$\Delta=\left\{\begin{array}{ll}
\neg box_0 \leadsto\neg box_1 &
(A_0 \vee E_0)\land \neg box_0\leadsto box_{1} \\
box_0\leadsto box_1 &
C_0\land \neg visible_0\leadsto visible_1 \\
\neg empty_0\leadsto\neg empty_1 &
C_0\land \neg visible_0\land empty_0\leadsto \neg visible_1\\
empty_0\leadsto empty_1 &
\neg box_1 \leadsto\neg box_2\\
\neg visible_0 \leadsto \neg visible_1 & ...\\
visible_0\leadsto visible_1 &  C_2\land \neg visible_2\land empty_2\leadsto \neg visible_3
\end{array}\right\}$

System Z will give a stratification in three strata where all persistence rules (of the form $v_t\leadsto v_{t+1}$ or $\neg v_t\leadsto \neg v_{t+1}$) are in the least specific stratum $\Delta_3$ (since they are tolerated by all the other rules). As seen before, $(A_0\vee E_0) \wedge \neg box_0\leadsto box_1$ describes an exception to the persistence of $\neg box$, just as $C_0\land \neg visible_0\leadsto visible_1$ describes an exception to the persistence of $\neg visible$ which leads us to place them in $\Delta_2$, the latter itself admits an exception described by rule $C_0\land \neg visible_0\land empty_0\leadsto \neg visible_1$ making it the most specific rule thus placed in $\Delta_1$ by System Z algorithm. At the end, we get:\\
$\Delta_1=\left\{\!\!\!\begin{array}{l}
C_{t}\land \neg visible_{t} \land empty_{t}\leadsto \neg visible_{t+1}\end{array}\right\}_{t\in\{0,1,2\}}$\\ 
$\Delta_2=\left\{\!\!\!\begin{array}{l}
C_{t}\land \neg visible_{t}\leadsto visible_{t+1}\\
(A_t\vee E_t) \wedge \neg box_t\leadsto box_{t+1} 
\end{array}\right\}_{\!\!\scriptsize t\in\{0,1,2\}}$ 
$\Delta_3=\left\{\!\!\!\begin{array}{l}
v_t\leadsto v_{t+1}\\ 
\neg v_t\leadsto \neg v_{t+1}\\
\end{array}\right\}_{\!\!\!\!\scriptsize\begin{array}{l}
t\in\{0,1,2\}\\v\in\{box,empty,visible\}\end{array}}$

Using lexicographic inference we get: $\neg box_0\entails_\Delta \neg box_1$ and $\neg box_0 \wedge (A_0 \vee E_0) \entails_\Delta box_1$, meaning that a priori if there was no box at time 0, there is no box at time 1, but knowing that either $A$ or $E$ has placed a box makes it more plausible that there is a box at time 1. 

Note that in this example, for the sake of simplicity, we want to make a closed world assumption (CWA) in order to express that the only possible way to change the variable $box$ (respectively $visible$) from false to true is the occurrence of $A$ or $E$ (respectively the performance of action $C$): $$\CWA=\{(\neg box_t \wedge box_{t+1}) \to (A_t \vee E_t), (\neg visible_t \wedge visible_{t+1}) \to C_t\}_{\scriptsize
t\in\{0,1,2\}}$$
From the set of default rules and the close world assumption, we can then obtain: $\{box_1\}\cup \CWA \entails_\Delta box_0$ meaning that the most plausible interpretation is that when there is a box at time point 1 it means that there was already a box at time 0. However, if we know that there were no box at time 0 then $\{box_1,\neg box_0\}\cup \CWA \entails_\Delta (A_0 \vee E_0)$ 
\end{exs}

We choose to use the lexicographic entailment in this paper, because \cite{BCDLP93} have shown that it is a powerful non-monotonic inference relation that satisfies the set of rational properties called  System P. The System P, introduced by Kraus, Lehmann and Magidor \cite{KLM90}, gathers properties that should follow rationally when one wants to deduce new inferences from a set of existing inference rules.
The following definition describes an agent epistemic states via the pieces of information that she believes.
\begin{defi}[Agent epistemic state and inference]
A user is represented by a tuple $B=(F,B_\L, B_\Delta)$ composed of a set $F\subseteq\L$ of formulas representing facts, and two sets $B_\L\subseteq \L$ and $B_\Delta$ respectively representing the strict and default rules known by the agent, the default rules of $B_\Delta$ are expressions of the form $\alpha\leadsto \beta$ with $\alpha,\beta\in\L$.

When $F\cup B_\L$ and $B_\Delta$ are both consistent\footnote{Here consistent is not used with the same meaning: for the propositional formulas it means classical logic consistency, while for the default rules base it means that the base can be stratified.}, the user is equipped with an inference relation between formulas of $\L$ denoted $\entails_B$ defined by:
$$\alpha \mbox{ $\entails_B$ } \beta \quad \mbox{ iff }\quad \begin{array}{|l}
\{\alpha\}\cup F\cup B_\L \mbox{ is consistent and}\\
\mbox{ for any Lex-preferred } (\alpha\wedge\bigwedge_{\varphi\in {F\cup B_\L}}\varphi) \mbox{-consistent subbase } A\in B_\Delta,\\
\hfill A\cup \{\alpha\}\cup F\cup B_\L\models \beta\end{array}$$ 

In the following, $\entails_B \varphi$ is a shortcut for $\top \entails_B \varphi$.
\end{defi}

In order to formally introduce  curiosity, we need to define awareness. This will be done by simply stating that an agent is aware of a variable if this variable appears in the facts contained in its epistemic state, and we assume that when an agent is aware of a variable then it becomes also aware of every variable of the strict or default rules of its epistemic state containing this variable (mimicking a kind of introspection). We use the notation $v\inn \varphi$ to express that the variable $v$ appears in the formula $\varphi$, this notation can be applied to variables of $\V_T$ as well as of $\V$. 

\begin{defi}[awareness]\label{defawareness} An agent represented by $B=(F,B_\L, B_\Delta)$ is \begin{itemize}
\item \emph{aware of a variable} $v\in \V$ if 
\begin{itemize} 
\item there is a formula $\varphi\in F$ s.t. $v\inn\varphi$ or
\item there is a formula $\varphi\in B_\L\cup str(B_\Delta)$ s.t. $v\inn\varphi$ and there is a variable $v'\inn\varphi$ of which the agent is aware; and
\end{itemize}
\item \emph{aware of a formula} $\varphi\in \L$ iff for any variable $v_t\inn\varphi$, the agent is aware of $v$.
\end{itemize}
\end{defi}

\begin{exs}{exboite}
Let us consider that the epistemic state of agent C is $(\emptyset,\CWA,\Delta=\Delta_1\cup\Delta_2\cup\Delta_3)$, in this case it does not know any fact, which means that it is not aware of anything. Assume now that at time point 1, our agent Cecilia comes to her office and sees a box on her desk, then the epistemic state of agent C is $(\{box_1\},\CWA,\Delta)$. In this state, it is aware that a box is on the desk, moreover by introspection its is aware of the possibility to open it due to rule concerning $C$, the possibility that Albert or Erwin are able to put it on the desk, the possibility that this box is $empty$ or that something inside of it could be $visible$.
\end{exs}

The following definition enables us to keep only formulas that do not concern a time point later than a given time point $t$, i.e., keep the formulas such that all their variables are indexed by time points no later than $t$. 
\begin{defi}[epistemic state until $t$]\label{defuntil}
Given an epistemic state $(F, B_\L,B_\Delta)$ and a time point $t\in [0,N]$, the \emph{epistemic state until $t$}, denoted $B\until{t}=(F\until{t},B_\L\until{t}, B_\Delta\until{t})$, is defined by:
$F\until{t}=\{\varphi \in F|\mbox{ for all } v_{t'}\inn\varphi, t'\leq t\}$, 
$B_\L\until{t}=\{\varphi \in B_\L|\mbox{ for all } v_{t'}\inn\varphi, t'\leq t\}$,
$B_\Delta\until{t}=\{\delta \in B_\Delta|\mbox{ for all } v_{t'}\inn str(\{\delta\}), t'\leq t\}$.
\end{defi}

In the following, we use $[\varphi]_{<t}$  (and respectively $[\varphi]_{\leq t}$, $[\varphi]_{> t}$, $[\varphi]_{\geq t}$ and $[\varphi]_{t}$) to denote that $\varphi$ is a formula containing only variables indexed by time points earlier than $t$ (resp. earlier than or equal to $t$, strictly later than $t$, later than or equal to $t$, equal to $t$). 
\begin{remark} For any formula $\varphi\in F\until{t}\cup B_\L\until{t}\cup str(B_\Delta\until{t})$, $[\varphi]_{\leq t}$ holds.
\end{remark}
An agent is curious about a formula at time point $t$ if according to its epistemic state until $t$ it is aware of this formula but it is not able to deduce its truth value at time $t$.
\begin{defi}[curiosity]\label{defcurious} An agent with state $B$ is  \emph{curious about $\varphi\in \L$} at $t\in T$ if, 
according to $B\until{t}$, it is aware of $\varphi$ and  
$\entailsnot_{B\until{t}} \varphi$ and 
$\entailsnot_{B\until{t}} \neg \varphi$.
\end{defi}

\begin{ex}\label{exboite2}
Coming back to Example \ref{exboite}, 
the epistemic state of C being $B=(\{box_1\}$, $\CWA$, $\Delta)$, its state at 0 is $B\until{0}=(\emptyset,\emptyset,\emptyset)$, meaning that at 0 it is aware of nothing, thus according to Definition \ref{defcurious} it is not curious about anything at 0.
In the epistemic state $B$, she first thinks that
$\{box_1\}\cup \CWA \entails_\Delta box_0$. However, she remembers that there was no box on her desk at time 0 before she left her office. Meaning that her epistemic state is now  $B'=(\{\neg box_0, box_1\},\CWA,\Delta)$ which enables her to draw the inference $\{\neg box_0,box_1\}\entails_{B'} (A_0 \vee E_0)$, however there is no way of knowing which of Albert or Bernard (or both) dropped off the box. More formally, we can say that Cecilia is curious about the possibility that Albert dropped off the box at time 0 because she is aware of this possibility and $\{\neg box_0,box_1\}\entailsnot_{B'} A_0$ and $\{\neg box_0,box_1\}\entailsnot_{B'} \neg A_0$.

Now if we consider that Albert told Cecilia that he placed a box on her desk at 0 before she entered her office. 
In that case, the epistemic state of Cecilia is $B''=(\{A_0$, $\neg box_0$, $box_1\},\CWA,\Delta)$, 
there is no more ambiguity as she knows who put it there, hence she is not curious about $A_0$.
\end{ex}


To define suspense, we propose for this formalization to use Baroni's description of \textit{primary suspense}~\cite{baroni_2007} which relies solely on temporal and belief factors. 
Baroni also describes other types of suspense involving different emotional components (empathy and identification with a protagonist for instance). These components affect suspense by strengthening the intensity of curiosity, and we will leave them for further study at the time being.
The following definition expresses that an agent feels suspense about a formula $\varphi$ when this agent is curious about it at time $t$, and thinks that it is not impossible for facts or events (below denoted $\psi$) to come to light and reveal the truth of $\varphi$ (satisfying curiosity about it at last).  
\begin{defi}[suspense] An agent represented by an epistemic state $B=(F,B_\L,B_\Delta)$ feels \emph{suspense about $\varphi\in\L$ at time point $t$} if 
\begin{enumerate}
\item according to $B$, the agent is curious about $\varphi$ at time $t$ and 
\item there is a formula $\psi\in\L$ such that $[\psi]_{>t}$ and $F\until{t} \cup B_\L \cup\{\psi\}$ consistent and
\item there is $t'>t$ s.t. either $\entails_{B'} \varphi_{t'}$ or $\entails_{B'} \neg\varphi_{t'}$ holds, with $B'=(F\cup\{\psi\},B_\L,B_\Delta)$.
\end{enumerate}
\end{defi}

\begin{ex}\label{exboite3}
In the context of Example \ref{exboite}, assume now that agent $C$ has the following epistemic state $B=(\{\neg box_0,box_1,\neg visible_1\},\CWA,\Delta)$. Here, at time 1 agent C is aware of the box, she is also aware that it is either empty or not, but has no way at this time to know which is true. Formally, $\entailsnot_B empty$ and $\entailsnot_B \neg empty$. Hence she is curious about the variable $empty$ at time point 1. Still according to Definition \ref{defawareness}, the agent is also aware of the formulas $(C_2\land \neg visible_2\leadsto visible_3)$ and $(C_2\land \neg visible_2\land empty_2\leadsto \neg visible_3)$. Meaning she is aware she will know the content of the box once she opens it. 

More precisely, the formula $\psi=C_2\wedge visible_3$ can be added to the facts of the epistemic state because $\{\neg box_0, box_1, \neg visible_1\} \cup \CWA \cup \{C_2 \wedge  visible_3\}$ is consistent. Now, $B'=(\{\neg box_0$, $box_1$, $\neg visible_1$, $C_2$, $visible_3\},\CWA,\Delta)$ yields $\entails_{B'} \neg empty_2$. Hence Cecilia feels suspense at time 1 about the truth value of $empty$.
\end{ex}
 
In order to formalize surprise, following \cite{DuPr23}, we propose to exploit our non-monotonic setting that enables agents to imagine several more or less plausible situations, i.e., enables them to incorporate new contradicting information by revising previous conclusions. This is required in order to avoid locking the agent in a state of total incomprehension. Surprise can then be defined by the occurrence of a formula that was unexpected but which is completely plausible.

\begin{defi}[surprise] An agent represented by $B=(F,B_\L,B_\Delta)$ is \emph{surprised at time $t$ about a formula $\varphi\in \L$} if 
$\varphi\in F\until{t}$ and
$B\until{t}$ is consistent ($\varphi$ occurred and it was not impossible) and
$B'=(F\until{t-1},B_\L\until{t},B_\Delta\until{t})$ is such that: $\entails_{B'} \neg\varphi$ ($\varphi$ was unexpected)
\end{defi}

\begin{exs}{exboite2}
Cecilia is surprised to find the box at time 1.
    Indeed, given the epistemic state $B=(\{\neg box_0,box_1\},\CWA,\Delta)$, before seeing the box at time 1, the persistence of $\neg box_0$ into $\neg box_1$ was the most plausible evolution. More formally, we can check that $box_1\in F$ and $B'=(\{\neg box_0\},\CWA\until{1},\Delta\until{1})$ is such that $\entails_{B'} \neg box_1$, and $B$ is consistent (hence $B\until{1}$ as well).  
\end{exs}

\section{Properties and graduality}\label{sec:properties}
In this section we show several simple properties relating the three emotions, moreover we establish the computational complexity of their detection. 

\subsection{Properties derived from the definitions}
An agent who knows nothing is aware of nothing.
\begin{prop}\label{propvide}
If the epistemic state of an agent has no fact, i.e,  $B=(\emptyset, B_\L,B_\Delta)$ then the agent is not aware of any variable. \end{prop}

\begin{proof} Even if $B_\L$ or $B_\Delta$ are non-empty, there is no awareness since no variable appears in $F$.
\end{proof}

This kind of agent is not curious nor able to feel suspense since curiosity requires awareness, and suspense requires curiosity.
\begin{corollary}[of Proposition \ref{propvide}] If the epistemic state of an agent has no fact, i.e., $B=(\emptyset, B_\L,B_\Delta)$, then the agent is not curious and does not feel suspense about any formula at any time point.
\end{corollary}

The following proposition states that an omniscient agent (i.e., an agent with complete information about the world) is never curious nor able to feel suspense.
\begin{prop}\label{propsaittout}
If the epistemic state $B=(F, B_\L,B_\Delta)$ of an agent admits only one most plausible interpretation in $\Omega$, then for any finite formula, there is no time point where the agent is curious or feels suspense about it.
\end{prop}
\begin{proof} In order to be curious, there should exist at least one variable whose truth value is unknown. Hence there should be at least two interpretations that are equally most plausible.
\end{proof}

Because surprise occurs when the agent expects something and then the opposite happens, it means that it is not curious about it (because the surprise makes it know it).
\begin{prop}\label{propsurprisnoncurieux}
Given an epistemic state $B$ of an agent, if the agent is surprised about $\varphi$ at time $t$ then the agent is not curious about $\varphi$ neither at time $t-1$ nor at time $t$.
\end{prop}

\begin{proof} Let us assume that $B=(F,B_\L,B_\Delta)$ is \emph{surprised at time $t$ about a formula $\varphi\in \L$}, it means that $\varphi \in F\until{t}$ and $B\until{t}$ is consistent and
$B'=(F\until{t-1},B_\L\until{t},B_\Delta\until{t})$ is such that: $\entails_{B'} \neg\varphi$. It means that the agent could infer the truth value of $\varphi$ at time $t-1$, hence she was not curious at $t-1$. Now since $\varphi\in F$ and $B\until{t}$ consistent then $\entails_{B\until{t}} \varphi$ hence she is not curious about it at $t$.
\end{proof}

This proposition shows that surprise and curiosity are antagonists in a given epistemic state, however we can imagine stories where the same event sequence may produce curiosity (e.g. by keeping some information hidden, namely the name of the murderer) when told in a given way and surprise when told differently (e.g. revealing this same information at start). 
 The following proposition shows the complexity class of the decision problems associated to awareness, curiosity, suspense and surprise. 
 
\begin{prop} Given an epistemic state $B$, a formula $\varphi\in\L$ and a time point $t\in T$,
\begin{itemize}
\item Deciding whether $B$ is aware of $\varphi$ at time point $t$ is linear
\item Deciding whether $B$ is curious or feels suspense or surprise about $\varphi$ at $t$ is $P^{\NP}$-complete.
\end{itemize}
\end{prop}
\begin{proof} In order to check awareness about a variable, it is enough to check membership of this variable to a set of formulas, which is linear in the size of the epistemic state, this process should be repeated for all the variables of a formula to check formula awareness. Concerning curiosity, in addition to a test of awareness, it uses two lexicographic inference tests which have been shown to be in $P^{\NP}$ by \cite{EiLu00}. Suspense requires a curiosity check and a consistency check of the strict part of the base $B$, which is a SAT problem hence ${\NP}$-complete. It then requires several lexicographic inferences in order to find the time point where $\entails_{B'}\varphi_{t'}$ or $\entails_{B'}\neg\varphi_{t'}$ holds. Surprise requires a consistency check of the default base of $B$ (which is a $P^{\NP}$-complete problem according to \cite{EiLu00})  and a lexicographic inference, hence the result.
\end{proof}

The complexity $P^{NP}$ of these decision problems is due to the use of the lexicographic inference in their definition. Note that the upper bound ($N$) on time steps could relieve the computational complexity as obtained in traditional STRIPS planning \cite{Bylander94} where the complexity of certain decision problems drops from PSPACE-complete to NP-complete. Note also that formulation of AI planning in answer set programming gives rise to similar complexity \cite{SPBS23}. 
 
\subsection{Towards defining measures}
For further characterizing narrative tension, we need to quantify the intensity of the emotions generated in an agent when listening to a story. This section is a first attempt towards this goal. We propose three definitions of the emotional intensity of curiosity, suspense and surprise. In the following definition we propose to rely on findings from Trabasso and Sperry \cite{trabasso_1985} as a heuristic in order to evaluate the intensity of the curiosity. We first define the causal graph associated with an epistemic state as the one relating variables of $V_T$ with the links induced by the default rules and strict rules of the epistemic state.
\begin{defi}[causal graph] The \emph{causal graph} $\G_B$ induced by an epistemic state $B=(F,B_\L,B_\Delta)$ is a pair $(V_B, E_B)$ with 
\begin{itemize}
\item $V_B=\{v_t\in \V_T | v_t \inn \varphi, \varphi\in F\cup B_\L\cup str(B_\Delta)\}$ is the set of vertices of $\G_B$
\item $E_B= \!\!\!\begin{array}[t]{l}
\{(v_t,v'_{t'}) \in \V_T\times \V_T| v_t \inn \alpha, v_{t'}\inn \beta, \alpha\leadsto \beta \in B_\Delta\} \quad \cup\\
\{(v_t,v'_{t'}) \in \V_B\times \V_B \quad |\quad \{l_t\}\cup F\cup B_\L \models l'_{t'} \mbox{ with } l_t\in\{v_t,\neg v_t\}, l_t'\in\{v'_{t'},\neg v'_{t'}\}  \}
\end{array}$
\end{itemize}
\end{defi}

We illustrate this definition on the epistemic state of Cecilia until time point 1.

\begin{exs}{exboite3} 
Considering $B=(\{\neg box_0,box_1,\neg visible_1\},\CWA,\Delta)$, the causal graph induced by $B\until{1}$ is shown in Figure \ref{figcausalgraph}. 
\end{exs}


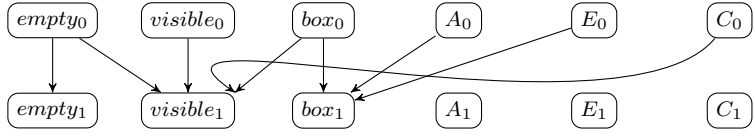
\begin{figure}
\begin{center}  \begin{tikzpicture}[xscale=1.8,yscale=1.2,>=stealth']
\foreach \n/\x/\y in {
{A_0}/4/0,
{E_0}/5/0,
{C_0}/6/0,
{box_0}/3/0,
{empty_0}/1/0,
{visible_0}/2/0,
{A_1}/4/-1,
{E_1}/5/-1,
{C_1}/6/-1,
{box_1}/3/-1,
{empty_1}/1/-1,
{visible_1}/2/-1
}
{\node[rectangle, rounded corners,draw] (\n) at (\x,\y) {\footnotesize$\n$};}
\foreach \orig/\dest in {
A_0/box_1,E_0/box_1,empty_0/visible_1,box_0/box_1,empty_0/empty_1,visible_0/visible_1, box_0/visible_1.north east}
{\draw[->] (\orig) edge  (\dest);}
\draw[->] (C_0) edge[out=240,looseness=1]  (visible_1.north east);
  \end{tikzpicture}
  \caption{Causal graph induced by the epistemic state $B\until{1}$}\label{figcausalgraph}
\end{center}
\end{figure}

\begin{defi}[curiosity intensity] Given an agent with state $B=(F,B_\L,B_\Delta)$ and \emph{curious about $\varphi\in \L$} at $t\in T$, her curiosity intensity level is $c_B(\varphi,t)=\sum_{\mbox{\small$v_{t'}\inn\varphi$}} deg(v_{t'})$ where $deg(x)$ is the degree of the node $x$ in the causal graph induced by $B$.
\end{defi}

\begin{exs}{exboite3} Given the epistemic state of Cecilia $B=(\{\neg box_0$, $box_1$, $\neg visible_1\}$, $\CWA$, $\Delta)$, she is the most curious about $\neg visible_1$ with intensity 5, denoted $c_B(\neg visible,1)=5$. Not that the degree of $visible_1$ is only four on Figure \ref{figcausalgraph} but there is a supplementary outgoing arc from $visible_1$ to $visible_2$ when considering $B$ instead of $B\until{1}$.
\end{exs}

Lets us now consider an example of suspense evolution. As previously explained, we base our definition only on beliefs and time. According to Baroni \cite{baroni_2007}, once curiosity is aroused then the suspense begins and lasts until it reaches a plateau, at which point it diminishes and gradually fades away, unless the suspense is resolved in the meantime. We propose to consider that the suspense profile of an agent is available under the form of a quadruplet $(\alpha,\beta,\gamma,\SMAX)$ where $\alpha$ is the duration before reaching the maximum of intensity $\SMAX$, $\beta$ the length of the plateau and $\gamma$ the descent duration (see Figure \ref{figtrapeze}).
\begin{figure}
\begin{tikzpicture}[>=stealth',xscale=.5,yscale=.15]
    \draw [<->,thick] (0,12) node (yaxis) [above] {Suspense intensity}
        |- node[below left] {0} (23,0) node (xaxis) [right] {$t$};
        \draw[line width=2pt,blue] (0,0) -- (6,0) node[below] {$t_0$} -- (6,3)  -- (12,10)   -- (18,10)  -- (21,0) -- (22,0);
        \draw[dashed, black] (0,10) node [left] {SMax} -- (12,10);
        \draw[dashed, black] (0,3) node [left] {$c$} -- (6,3);
        \draw[dashed, black] (12,-0.2)  -- (12,10);
        \draw[dashed, black] (18,-0.2)  -- (18,10);
        \draw[<->] (6,-3.5) -- node[below] {$\alpha$} (12,-3.5) ;
        \draw[<->] (12.1,-3) -- node[below] {$\beta$} (17.9,-3) ; 
        \draw[<->] (18,-3.5) -- node[below] {$\gamma$} (21,-3.5);
\end{tikzpicture}
\caption{Suspense intensity along time ($c$ being the level of curiosity felt at time $t_0$)}\label{figtrapeze}
\end{figure}
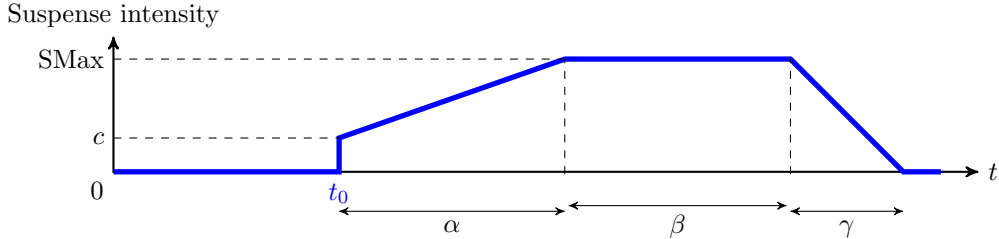
Thus, suspense intensity is a function of the curiosity intensity at the time it is first felt and of the duration between its triggering and its resolution. 
\begin{defi}[suspense intensity] Given an epistemic state $B=(F,B_\L,B_\Delta)$ and a suspense profile $p=(\alpha,\beta,\gamma,\SMAX)$ then the intensity of the suspense feeling at $t$ is $$s^p_B(\varphi,t)= \left\{\begin{array}{ll}
0 &\mbox{ if $t<t_0$}\\
\frac{\SMAX-c}{\alpha}(t -t_0)+ c&\mbox{ if $t\in[t_0,t_0+\alpha]$}\\
\SMAX &\mbox{ if $t\in[t_0+\alpha,t_0+\alpha+\beta]$}\\
-\frac{\SMAX}{\gamma}(t-t_0-\alpha-\beta) + \SMAX &\mbox{ if $t\in[t_0+\alpha+\beta,t_0+\alpha+\beta+\gamma]$}\\
 0&\mbox{ if $t\geq t_0+\alpha+\beta+\gamma$}
\end{array}\right.$$
where $t_0$ is the earliest time where the agent was curious about $\varphi$ and $c=c_B(\varphi,t_0)$ is the curiosity intensity at $t_0$.
\end{defi}

Note that in this definition, the suspense intensity  may only vary according to the profile of the agent and the duration. A more refined way to handle this would be to define a decreasing persistence of awareness, enabling the agent to forget a variable after some delay, it would be in accordance with the common knowledge that the suspense should be revived from time to time.

Concerning surprise intensity about a formula $\varphi$, we propose to adopt the point of view of Shackle \cite{Shackle1961} as done in \cite{DuPr23}, by assimilating it to the degree of impossibility of $\varphi$ (or equivalently the possibility degree of $\neg \varphi$). It amounts to finding the most specific rule that is violated by $\varphi \cup F\cup B_{\L}$, the more specific this rule, the more surprising $\varphi$ becomes\footnote{Such a definition is classical in possibility theory, and more specifically in the context where a default rule $\alpha \leadsto \beta$ is interpreted as a constraint on a possibility measure $\Pi$ see e.g., \cite{BDP1997}. This constraint being $\Pi(\alpha\wedge \beta)>\Pi(\alpha\wedge \neg \beta)$  expressing that when $\alpha$ is true, having $\beta$ true is more possible than $\beta$ false.}.

\begin{defi}[surprise intensity] Given an epistemic state $B=(F,B_\L,B_\Delta=\Delta_1\ldots\Delta_n)$ where there is a surprise at time $t$, the surprise intensity is $surp_B(\varphi,t)=n-i$, where $i$ is the most specific strata level, s.t. there is a rule $\alpha\leadsto\beta\in \Delta_i$ with $\{\varphi\wedge\alpha\}\cup F\cup B_\L\models \neg \beta$.
\end{defi}

The definitions of this section are a first step towards being able to compare stories with respect to the intensity of emotions they generate in the agent listening to them.  

\section{Conclusion}
This paper aims at providing a unified framework in which the three emotions at the
heart of narrative tension, namely curiosity, surprise, and suspense are formalized and their relationships clarified. This
framework is built on non-monotonic reasoning 
for representing compactly the default behavior of the world and also for simulating the
reasoning of an agent in front of a story. The use of non-monotonic reasoning
induces a cost in complexity: the detection problems associated with the three
emotions are in $P^{NP}$ (due to the use of lexicographic inference). For each
of the three emotions, we describe methods to evaluate their intensity.

While we illustrated our formalization by adopting the point of view of a single agent in a chronological story for the sake of clarity, it does not preclude its adaptability for storytelling using other points of views such as an extradiegetic narrator disclosing knowledge to the listener through a discourse that does not reflect the timeline of the story. 

To operationalize this model, we plan to investigate different frameworks that are equipped with solvers namely PDDL planning, Linear logic with Ceptre \cite{Martens15} and propositional default logic with TouIST \cite{touist}. Moreover, the inherent growing complexity of this problem for scaling to complex narratives requires further study about the granularity of story events, for instance inspired by discussions about the representation of causality~\cite{Mazlack_2004}.

\bibliography{bib}

\begin{thebibliography}{10}

\bibitem{adam_2009}
Carole Adam, Andreas Herzig, and Dominique Longin.
\newblock A logical formalization of the {{OCC}} theory of emotions.
\newblock {\em Synthese}, 168(2):201--248, May 2009.

\bibitem{AGM85}
Carlos~E Alchourr{\'o}n, Peter G{\"a}rdenfors, and David Makinson.
\newblock On the logic of theory change: Partial meet contraction and revision
  functions.
\newblock {\em The journal of symbolic logic}, 50(2):510--530, 1985.

\bibitem{aucher2013undecidability}
Guillaume Aucher and Thomas Bolander.
\newblock Undecidability in epistemic planning.
\newblock In F.~Rossi, editor, {\em Proc. of Int. Joint Conf. on Artificial
  Intelligence (IJCAI'2013)}, pages 27--–33, Beijing, China, 2013.

\bibitem{barber_2010}
Heather Barber and Daniel Kudenko.
\newblock Generation of dilemma-based interactive narratives with a changeable
  story goal.
\newblock In {\em 2nd International Conference on INtelligent TEchnologies for
  interactive enterTAINment}. ICST, 5 2010.

\bibitem{baroni_2007}
Rapha{\"e}l Baroni.
\newblock {\em {La tension narrative: suspense, curiosit{\'e} et surprise}}.
\newblock {Po{\'e}tique}. {\'E}d. du Seuil, Paris, 2007.

\bibitem{BCDLP93}
Salem Benferhat, Claudette Cayrol, Didier Dubois, J{\'e}r{\^o}me Lang, and
  Henri Prade.
\newblock Inconsistency management and prioritized syntax-based entailment.
\newblock In {\em Proc. of Int. Joint. Conf. on Artificial Intelligence
  (IJCAI'93)}, volume~93, pages 640--645, 1993.

\bibitem{BDP1997}
Salem Benferhat, Didier Dubois, and Henri Prade.
\newblock Nonmonotonic reasoning, conditional objects and possibility theory.
\newblock {\em Artif. Intell.}, 92(1-2):259--276, 1997.

\bibitem{bolander2011epistemic}
Thomas Bolander and Mikkel~Birkegaard Andersen.
\newblock Epistemic planning for single-and multi-agent systems.
\newblock {\em Journal of Applied Non-Classical Logics}, 21(1):9--34, 2011.

\bibitem{bonoli_2008}
Lorenzo Bonoli.
\newblock {Rapha{\"e}l Baroni, La tension narrative. Suspense, curiosit{\'e},
  surprise, Paris, Seuil, 2007}.
\newblock {\em Cahiers de Narratologie. Analyse et th{\'e}orie narratives}, 14,
  February 2008.

\bibitem{bosser_linear_2010}
Anne-Gwenn Bosser, Marc Cavazza, and Ronan Champagnat.
\newblock Linear {Logic} for {Non}-{Linear} {Storytelling}.
\newblock {\em ECAI 2010}, pages 713--718, 2010.
\newblock Publisher: IOS Press.

\bibitem{Bylander94}
Tom Bylander.
\newblock The computational complexity of propositional strips planning.
\newblock {\em Artificial Intelligence}, 69(1-2):165--204, 1994.

\bibitem{cardona-rivera_2012}
Rogelio~E. Cardona-Rivera, Bradley~A. Cassell, Stephen~G. Ware, and R.~Michael
  Young.
\newblock Indexter : {{A Computational Model}} of the {{Event-Indexing
  Situation Model}} for {{Characterizing Narratives}}, 2012.

\bibitem{Rivera_Jhala_Porteous_Young_2024}
Rogelio~E. Cardona-Rivera, Arnav Jhala, Julie Porteous, and R.~Michael Young.
\newblock The story so far on narrative planning.
\newblock {\em Proceedings of the International Conference on Automated
  Planning and Scheduling}, 34(1):489--499, May 2024.

\bibitem{carroll_2007}
No{\"e}l Carroll.
\newblock Narrative closure.
\newblock {\em Philosophical Studies}, 135(1):1--15, August 2007.

\bibitem{cheong_2015}
Yun-Gyung Cheong and R.~Michael Young.
\newblock Suspenser: {{A Story Generation System}} for {{Suspense}}.
\newblock {\em IEEE Transactions on Computational Intelligence and AI in
  Games}, 7(1):39--52, March 2015.

\bibitem{DoLM07}
Christophe Dousson and Pierre Le~Maigat.
\newblock Chronicle recognition improvement using temporal focusing and
  hierarchization.
\newblock In {\em IJCAI}, volume~7, pages 324--329. Citeseer, 2007.

\bibitem{DupindeSaintCyr08}
Florence {Dupin de Saint-Cyr}.
\newblock Scenario {{Update Applied}} to {{Causal Reasoning}}.
\newblock In {\em Principles of {{Knowledge Representation}} and {{Reasoning}}:
  {{Proceedings}} of the 11th {{International Conference}}, {{KR}} 2008}, pages
  188--197, January 2008.

\bibitem{DHLM20}
Florence Dupin~de Saint-Cyr, Andreas Herzig, J{\'e}r{\^o}me Lang, and Pierre
  Marquis.
\newblock {Reasoning About Action and Change}.
\newblock In Pierre Marquis, Odile Papini, and Henri Prade, editors, {\em {A
  Guided Tour of Artificial Intelligence Research}}, volume 1 / 3 of {\em
  Knowledge Representation, Reasoning and Learning}, pages 487--518. {Springer
  International Publishing}, May 2020.

\bibitem{DuLa11}
Florence {Dupin De Saint-Cyr} and J{\'e}r{\^o}me Lang.
\newblock Belief extrapolation (or how to reason about observations and
  unpredicted change).
\newblock {\em Artificial Intelligence}, 175(2):760--790, February 2011.

\bibitem{DuPr23}
Florence {Dupin de Saint-Cyr} and Henri Prade.
\newblock Belief revision and incongruity: Is it a joke?
\newblock {\em Journal of Applied Non-Classical Logics}, 33(3-4):467--494,
  October 2023.

\bibitem{EiLu00}
Thomas Eiter and Thomas Lukasiewicz.
\newblock Default reasoning from conditional knowledge bases: Complexity and
  tractable cases.
\newblock {\em Artificial Intelligence}, 124(2):169--241, 2000.

\bibitem{ekman_1992}
Paul Ekman.
\newblock An argument for basic emotions.
\newblock {\em Cognition and Emotion}, 6(3-4):169--200, May 1992.

\bibitem{ely_2015}
Jeffrey Ely, Alexander Frankel, and Emir Kamenica.
\newblock Suspense and {{Surprise}}.
\newblock {\em Journal of Political Economy}, 123(1):215--260, February 2015.

\bibitem{Finger87}
Joseph~Jeffrey Finger.
\newblock {\em Exploiting constraints in design synthesis}.
\newblock PhD thesis, Stanford University, Stanford, CA, 1987.

\bibitem{genette1979}
Gerard Genette.
\newblock {\em Nouveau Discours du Récit}.
\newblock Seuil, 1983.

\bibitem{GoPe91}
Mois{\'e}s Goldszmidt and Judea Pearl.
\newblock On the consistency of defeasible databases.
\newblock {\em Artificial Intelligence}, 52(2):121--149, 1991.

\bibitem{green_2000}
Melanie Green and Timothy Brock.
\newblock The {{Role}} of {{Transportation}} in the {{Persuasiveness}} of
  {{Public Narrative}}.
\newblock {\em Journal of personality and social psychology}, 79:701--21,
  November 2000.

\bibitem{halpern_2001}
Joseph~Y. Halpern.
\newblock Alternative {{Semantics}} for {{Unawareness}}.
\newblock {\em Games and Economic Behavior}, 37(2):321--339, November 2001.

\bibitem{Hintikka62}
Kaarlo Jaakko~Juhani Hintikka.
\newblock {\em Knowledge and belief: An introduction to the logic of the two
  notions}.
\newblock Cornell University Press, Ithaca and London, 1962.

\bibitem{KaMe91}
Hirofumi Katsuno and Alberto~O Mendelzon.
\newblock On the difference between updating a knowledge base and revising it.
\newblock {\em KR}, 91:387--394, 1991.

\bibitem{KLM90}
Sarit Kraus, Daniel Lehmann, and Menachem Magidor.
\newblock Nonmonotonic reasoning, preferential models and cumulative logics.
\newblock {\em Artificial intelligence}, 44(1-2):167--207, 1990.

\bibitem{Lehmann95}
Daniel Lehmann.
\newblock Another perspective on default reasoning.
\newblock {\em Annals of mathematics and artificial intelligence}, 15:61--82,
  1995.

\bibitem{LoCa2007}
Emiliano Lorini and Cristiano Castelfranchi.
\newblock The cognitive structure of surprise: looking for basic principles.
\newblock {\em Topoi}, 26(1):133--149, 2007.

\bibitem{lorini_2011}
Emiliano Lorini and Fran{\c c}ois Schwarzentruber.
\newblock A logic for reasoning about counterfactual emotions.
\newblock {\em Artificial Intelligence}, 175(3):814--847, March 2011.

\bibitem{Martens15}
Chris Martens.
\newblock Ceptre: A language for modeling generative interactive systems.
\newblock In {\em Proceedings of the AAAI Conference on Artificial Intelligence
  and Interactive Digital Entertainment}, volume~11, pages 51--57, 2015.

\bibitem{Mazlack_2004}
Lawrence~J. Mazlack.
\newblock Granular causality speculations.
\newblock In {\em IEEE Annual Meeting of the Fuzzy Information, 2004.
  Processing NAFIPS '04.}, volume~2, pages 690--695 Vol.2, 2004.

\bibitem{McCarthy77}
John McCarthy.
\newblock Epistemological problems of artificial intelligence.
\newblock In {\em Proc. Int. Joint Conf on Artificial Intelligence (IJCAI'77)},
  pages 1038--1044. Elsevier, 1977.

\bibitem{McHa69}
John McCarthy and Patrick~J Hayes.
\newblock Some philosophical problems from the standpoint of artificial
  intelligence.
\newblock {\em {M}achine {I}ntelligence}, 4:463--502, 1969.

\bibitem{modica1994awareness}
Salvatore Modica and Aldo Rustichini.
\newblock Awareness and partitional information structures.
\newblock {\em Theory and decision}, 37:107--124, 1994.

\bibitem{norling_2003}
Emma Norling.
\newblock Capturing the quake player: Using a {{BDI}} agent to model human
  behaviour.
\newblock In {\em Proceedings of the Second International Joint Conference on
  {{Autonomous}} Agents and Multiagent Systems}, pages 1080--1081, Melbourne
  Australia, July 2003. ACM.

\bibitem{ortony_2022}
Andrew Ortony, Gerald~L Clore, and Allan Collins.
\newblock {\em The cognitive structure of emotions}.
\newblock Cambridge university press, 2022.

\bibitem{Pearl1990}
Judea Pearl.
\newblock System {Z}: A natural ordering of defaults with tractable
  applications to nonmonotonic reasoning.
\newblock In {\em Proc. 3rd Conf. on Theoretical Aspects of Reasoning about
  Knowledge}, pages 121--135, 1990.

\bibitem{pizzi_2007}
David Pizzi, Fred Charles, Jean-Luc Lugrin, and Marc Cavazza.
\newblock Interactive {{Storytelling}} with {{Literary Feelings}}, April 2007.

\bibitem{plutchik_1980}
Robert Plutchik.
\newblock A general psychoevolutionary theory of emotion.
\newblock In {\em Theories of emotion}, pages 3--33. Elsevier, 1980.

\bibitem{propp_1968}
Vladimir Propp.
\newblock {\em Morphology of the {{Folktale}}: {{Second Edition}}}.
\newblock University of Texas Press, 1968.

\bibitem{Reiter80}
Raymond Reiter.
\newblock A logic for default reasoning.
\newblock {\em Artificial intelligence}, 13(1-2):81--132, 1980.

\bibitem{rivera-villicana_2016a}
Jessica {Rivera-Villicana}, Fabio Zambetta, James Harland, and Marsha Berry.
\newblock Using {{BDI}} to {{Model Players Behaviour}} in an {{Interactive
  Fiction Game}}.
\newblock In Frank Nack and Andrew~S. Gordon, editors, {\em Interactive
  {{Storytelling}}}, volume 10045, pages 209--220. Springer International
  Publishing, Cham, 2016.

\bibitem{Shackle1961}
George Lennox~Sharman Shackle.
\newblock {\em Decision, Order and Time in Human Affairs}.
\newblock (2nd edition), Cambridge University Press, UK, 1961.

\bibitem{touist}
Khaled Skander~Ben Slimane, Alexis Comte, Olivier Gasquet, Abdelwahab Heba,
  Olivier Lezaud, Frederic Maris, and Ma{\"{e}}l Valais.
\newblock Twist your logic with touist.
\newblock {\em CoRR}, abs/1507.03663, 2015.

\bibitem{SPBS23}
Tran~Cao Son, Enrico Pontelli, Marcello Balduccini, and Torsten Schaub.
\newblock Answer set planning: a survey.
\newblock {\em Theory and Practice of Logic Programming}, 23(1):226--298, 2023.

\bibitem{sternberg_2001}
Meir Sternberg.
\newblock How {{Narrativity Makes}} a {{Difference}}.
\newblock {\em Narrative}, 9(2):115--122, 2001.

\bibitem{thue_2006}
David Thue and Vadim Bulitko.
\newblock Modelling goal-directed players in digital games.
\newblock In {\em Proceedings of the {{AAAI Conference}} on {{Artificial
  Intelligence}} and {{Interactive Digital Entertainment}}}, volume~2, pages
  86--91, 2006.

\bibitem{trabasso_1985}
Tom Trabasso and Linda~L Sperry.
\newblock Causal relatedness and importance of story events.
\newblock {\em Journal of Memory and Language}, 24(5):595--611, October 1985.

\bibitem{vanbenthem_2010}
Johan Van~Benthem and Fernando~R Vel{\'a}zquez-Quesada.
\newblock The dynamics of awareness.
\newblock {\em Synthese}, 177:5--27, 2010.

\bibitem{Winslett90}
Marianne Winslett.
\newblock {\em Updating Logical Databases}.
\newblock Cambridge University Press, 1990.

\end{thebibliography}
\bibliographystyle{plain}


\end{document}